\documentclass{ecai}
\usepackage{graphicx}
\usepackage{amsmath}
\usepackage{booktabs}
\usepackage[linesnumbered,ruled]{algorithm2e}
\usepackage{appendix}
\usepackage{url}
\usepackage{multirow}
\usepackage{epstopdf}
\usepackage{amsthm}
\newtheorem{remark}{Remark}

\newtheorem{theorem}{Theorem}

\usepackage{enumitem}
\usepackage{array}
\usepackage{subfigure}
\usepackage{color}
\usepackage{wasysym}

\newtheorem{definition}{Definition}

\newcommand{\BibTeX}{B\kern-.05em{\sc i\kern-.025em b}\kern-.08em\TeX}

\begin{document}

\begin{frontmatter}

\paperid{863}

\title{Learning Order Forest for \\ Qualitative-Attribute Data Clustering}

\author[1]{\fnms{Mingjie}~\snm{Zhao}}
\author[1]{\fnms{Sen}~\snm{Feng}}
\author[1,6]{\fnms{Yiqun}~\snm{Zhang}\thanks{Corresponding Author. Email: yqzhang@gdut.edu.cn}}
\author[2,3,6]{\fnms{Mengke}~\snm{Li}}
\author[4,5,6]{\fnms{Yang}~\snm{Lu}}
\author[6]{\fnms{Yiu-Ming}~\snm{Cheung}}

\address[1]{School of Computer Science and Technology, Guangdong University of
Technology, Guangzhou, China}
\address[2]{Guangdong Laboratory of Artificial Intelligence and Digital Economy (SZ), Shenzhen, China}
\address[3]{School of Computer Science and Software Engineering, Shenzhen University, Shenzhen, China}
\address[4]{Fujian Key Laboratory of Sensing and Computing for Smart City, School of Informatics, Xiamen University, China}
\address[5]{Key Laboratory of Multimedia Trusted Perception and Efficient Computing,
Ministry of Education of China, Xiamen University, China}
\address[6]{Department of Computer Science, Hong Kong Baptist University, Hong Kong, China}

\begin{abstract}
Clustering is a fundamental approach to understanding data patterns, wherein the intuitive Euclidean distance space is commonly adopted. However, this is not the case for implicit cluster distributions reflected by qualitative attribute values, e.g., the nominal values of attributes like symptoms, marital status, etc. This paper, therefore, discovered a tree-like distance structure to flexibly represent the local order relationship among intra-attribute qualitative values. That is, treating a value as the vertex of the tree allows to capture rich order relationships among the vertex value and the others. To obtain the trees in a clustering-friendly form, a joint learning mechanism is proposed to iteratively obtain more appropriate tree structures and clusters. It turns out that the latent distance space of the whole dataset can be well-represented by a forest consisting of the learned trees. Extensive experiments demonstrate that the joint learning adapts the forest to the clustering task to yield accurate results. Comparisons of 10 counterparts on 12 real benchmark datasets with significance tests verify the superiority of the proposed method. Source code of the proposed method is available at~\cite{code}.
\end{abstract} 

\end{frontmatter}
\vspace{-8pt}
\section{Introduction}\label{sct:Intro}
Datasets composed of multi-valued qualitative attributes (also known as categorical or nominal attributes) are ubiquitous in cluster analysis tasks~\cite{intro13, ecai-2023-2, ecai-2023-3}, for instance, the clustering of clients, patients, and so on. Unlike a numerical attribute with all its values distributed on an Euclidean distance axis, the qualitative values of a categorical attribute cannot reflect its distance structure. For example, there are three possible values \{driver, lawyer, nurse\} for attribute ``occupation'', but their optimal numerical embedding on a distance axis is unknown. Therefore, most related works are dedicated to mining this implicit distance structure~\cite{intro2,intro5,intro10}, and can be roughly divided into: 1) Distance measures and 2) Distance learning methods~\cite{intro7}, according to whether they connect to the downstream clustering tasks.

It is noteworthy that distance measures for qualitative values, e.g., Hamming distance~\cite{hdm}, simply perform a boolean distance measurement based on whether two values are the same or not. Although subsequent measures~\cite{adm, abdm, cde} introduce various data statistical information to improve distance discrimination, they still treat each inter-value distance in isolation without considering the overall distance structure of all possible values. To address this issue, information entropy is introduced to effectively couple the possible values of an attribute, and an entropy-based distance metric~\cite{lsm} is formed accordingly to more appropriately quantify the distances. Recently, more distance metrics~\cite{ebdmjournal, udm, adc} attempt to orderly embed possible values into a distance axis to obtain a distance structure similar to that of numerical attributes. However, their value order relies on the explicit semantic order of the values, e.g., \{strong\_accept, clear\_accept, weak\_accept\} for an ordinal attribute ``review recommendation'', which is often unavailable when dealing with nominal attributes.

\begin{figure}[!t]	
\centerline{\includegraphics[width=3.5in]{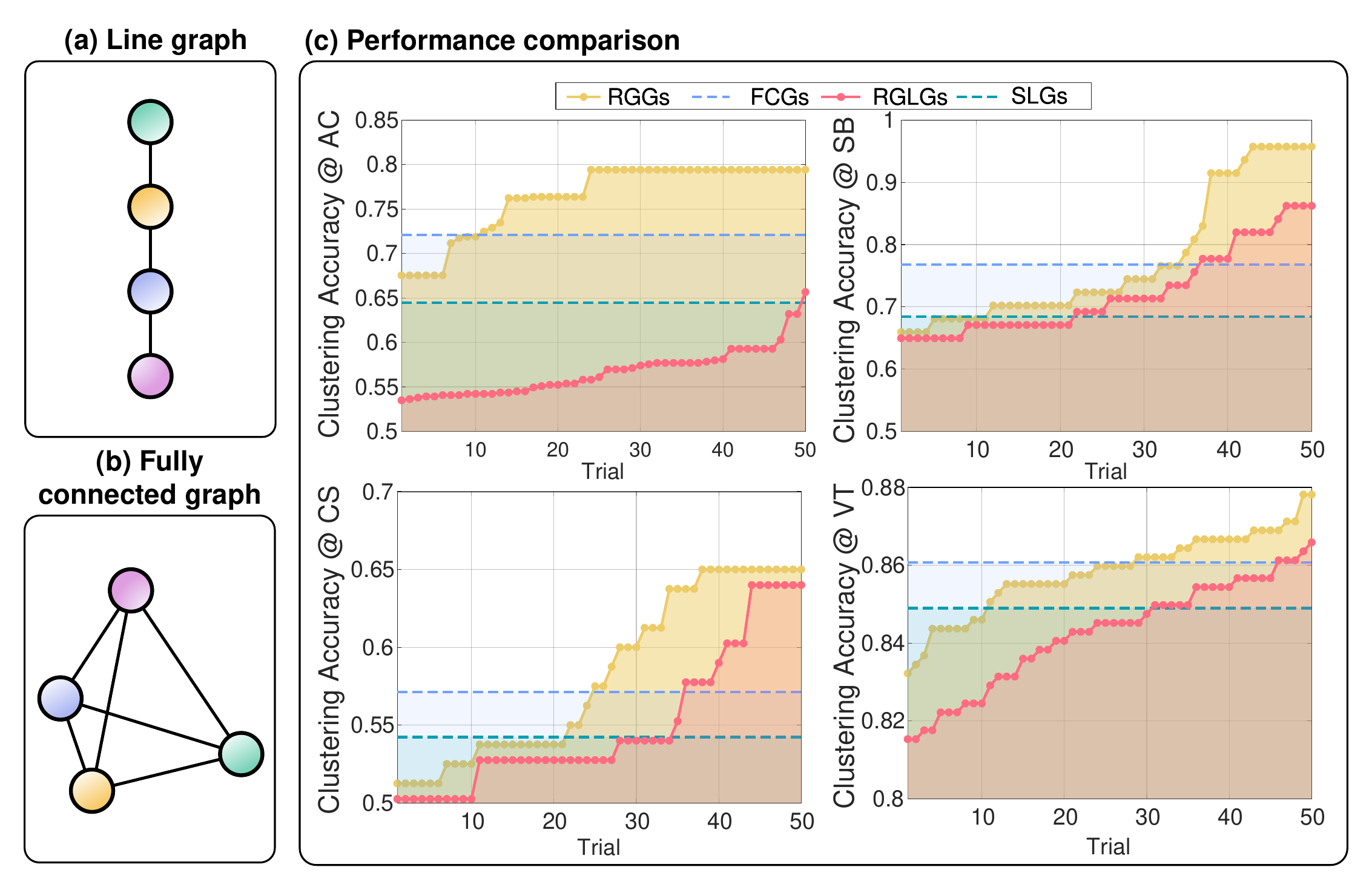}}

\caption{An intuitive comparison of clustering performance by adopting different types of distance structures. (a) and (b) demonstrate the typical line graph and fully connected graph. (c) demonstrates the $k$-modes~\cite{kmd} clustering performance with the following distance structures: 1) Randomly Generated Graphs (RGGs, not necessarily fully connected but ensure all attribute values are connected), 2) Fully Connected Graphs (FCGs), 3) Randomly Generated Line Graphs (RGLGs), and 4) Semantic Line Graphs (SLGs, arrange possible values in the graph according to their semantic order). The RGGs and RGLGs involving randomization are implemented 50 times, and the clustering accuracy is sorted for better visualization.}	
\label{fig:toy_example}	
\end{figure}

Distance learning methods focusing on connecting distance measurements and clustering tasks have received more attention in recent years, as they can often obtain distance structures that are more suitable for clustering. An early attempt~\cite{ocil} models sample-cluster similarity as the occurrence probability of possible values in clusters. Later, approaches that directly model the distance space have been proposed, including kernel-based~\cite{untie} and graph-based~\cite{dlc, HDC} distance metric learning. However, they are all based on specific hypotheses, e.g., specific kernels can well-represent the distance metric, or the distance metric follows a graph structure of possible values inspired by specific domain knowledge. 

Benefiting from the universality of the graph, the graph-based learning approaches~\cite{dlc, HDC, het2hom, adc, KDD'24} are proven to achieve more competitive clustering performance. More specifically, a graph has been adopted to represent the relationship among possible values of an attribute. For the ordinal attribute values with explicit semantic order, they adopt a line graph shown in Figure~\ref{fig:toy_example}~(a) whereby the weights of edges are learned to indicate the distances. For the nominal attribute values without semantic order, they use a fully connected graph shown in Figure~\ref{fig:toy_example}~(b) to facilitate distance learning. Nevertheless, a pair of coupled thorny problems still lies ahead: reasonable prior knowledge is the premise of effective distance learning whilst the data knowledge is usually obtained by observing data distribution under well-defined distance metrics.

The limitations brought by the prior knowledge can be fully verified by Figure~\ref{fig:toy_example}~(c). 
It can be seen that clustering under the two types of random graphs, i.e., RGGs and RGLGs, is significantly more promising to obtain higher accuracy compared to FCGs and SLGs. Moreover, RGGs obviously outperform RGLGs probably because RGGs do not overly restrict the relationship among attribute values to follow an order, thus laying the foundation for obtaining the latent optimal relationship through randomization. The above observations provide two hints: 1) A higher degree of topological freedom for the distance structure brings better clustering results, and 2) Explicit semantic order may not be optimal for clustering. Hence, how to obtain the optimal distance structure w.r.t. certain clustering tasks without relying on prior knowledge of the value relationship is crucial for breaking through the current clustering performance.


In this paper, a new qualitative data learning paradigm that performs Clustering with Order Forest learning (COForest) is proposed. The learning process is no longer limited to adjusting the distance between values under the hypothesized value graph, but allows both graph structure and distances to be jointly learned with clustering. It learns by iteratively: 1) Inferring graph structures w.r.t. the current data partition, and 2) Performing clustering using the graph distance structure to more appropriately obtain data partition. Since the inferred graphs are minimal spanning trees, they can concisely and flexibly represent the relationship among possible values. It turns out that the learning processes repeatedly improve the upper bound and approach it, thus bypassing sub-optimal solutions and achieving superior clustering accuracy. 
Main contributions of this work are summarized into three-fold:
\vspace{-5pt}
\begin{itemize}
    \item A new insight is introduced that there exists an optimal latent graph w.r.t. certain clustering tasks in representing the distance structure of a qualitative attribute, and the graph should be flexibly determined without being restricted by prior knowledge.
    \item COForest is proposed to iteratively optimize the distance structures and clusters to circumvent sub-optimal solutions. Compared with the existing approaches that only tune distances under a given topology, COForest further allows the reconstruction of the topology and thus brings a higher degree of learning freedom.
    \item Comprehensive experimental evaluations including significance tests, ablation studies, and qualitative visual comparisons, have been conducted to demonstrate the superiority of thoroughly learning distance structures without prior knowledge bias.
\end{itemize}

\vspace{-8pt}
\section{Propose Method}
\label{sct:Method}

\subsection{Problem Formulation}
\label{sct:Method_PF}
The problem of categorical data clustering with distance learning is formulated below. Given a categorical dataset $X=\{\mathbf{x}_1,\mathbf{x}_2,...,\mathbf{x}_n\}$ with $n$ data samples. Each sample $\mathbf{x}_i$ can be denoted as an $l$-dimensional row vector $\mathbf{x}_i=[x_{i,1},x_{i,2},..,x_{i, l}]^\top$ represented by $l$ attributes $A=\{\mathbf{a}_1,\mathbf{a}_2,...,\mathbf{a}_l\}$. Each attribute $\mathbf{a}_r$ can be denoted as a column vector $\mathbf{a}_r=[x_{1,r},x_{2,r},...,x_{n,r}]$ composed of the $r$-th values of all the $n$ samples, where the $n$ values can be viewed as sampled from a limited number of possible values $V_r=\{v_{r,1},v_{r,2},...,v_{r,o_r}\}$ with $o_r$ indicating the number of possible values of $\mathbf{a}_r$, and $v_{r,g}$ indicating the $g$-th possible value of $\mathbf{a}_r$. 

Partitional clustering aims to partition $X$ into $k$ non-overlapping sample subsets $C=\{C_1,C_2,...,C_k\}$ with the objective of minimizing the intra-subset dissimilarity~\cite{sup1,sup4,sup10}, which is conventionally expressed as
\begin{equation}
    \label{eq:obj_pre}	L(\mathbf{Q})=\sum_{j=1}^k\sum_{i=1}^nq_{i,j}\cdot\Gamma(\mathbf{x}_i,C_j),
\end{equation}
where $\mathbf{Q}$ is an $n\times k$ matrix with its $(i,j)$-th entry $q_{i,j}\in\{0,1\}$ indicating the affiliation between sample $\mathbf{x}_i$ and cluster $C_j$, and each row of $\mathbf{Q}$ (e.g., the $i$-th row) satisfies $\sum_{j=1}^{k}q_{i,j}=1$. During clustering, the values of $q_{i,j}$ are determined by
\begin{equation}
	\label{eq:qim_pre}
	q_{i,j}=\left\{
	\begin{array}{ll}
		1,  & \text{if}\ j=\arg\min\limits_y\Gamma(\mathbf{x}_i,C_y)\\
		0,  & \text{otherwise.}\\
	\end{array}\right.
\end{equation}
The sample-cluster dissimilarity $\Gamma(\mathbf{x}_i,C_j)$ is collectively reflected by the value-level dissimilarities of different attributes by
\begin{equation}
	\label{eq:dist_m}
	\Gamma(\mathbf{x}_i,C_j)=\sum_{r=1}^{l}\gamma(x_{i,r},C_{j,r}),
\end{equation}
where $\gamma(x_{i,r}, C_{j,r})$ is the dissimilarity between the sample $\mathbf{x}_i$ and the cluster $C_j$ from the perspective of attribute $\mathbf{a}_r$. 

Since the relationship among possible values remains to be defined, it is not straightforward to compute the dissimilarity between a sample and a sample set $C_j$. Therefore, the key to this work is how to reasonably define the distance $\gamma(x_{i,r}, C_{j,r})$ in the attribute aspect. In the following, we first show how to flexibly model the value-level relationship, based on which the corresponding distance metric is defined. Then the joint learning scheme is proposed to make the defined distance structure learnable with clustering.
\vspace{-8pt}
\subsection{Order Forest Construction}
\label{sct:Method_OFC}

As mentioned in Section~\ref{sct:Intro}, both the line graph and the fully connected graph have significant shortcomings in representing the distance structure of an attribute. That is, a line graph forces the relationship among all the possible values to be an order, and its effectiveness heavily relies on the prior order knowledge of the possible values. By contrast, a fully connected graph constructs multiple paths for each pair of possible values, and thus cannot concisely and exactly reflect the relationship among possible values.

Therefore, we propose to construct order forest $M=\{M_1, M_2,\ldots, M_{l}\}$ where $M_r$ is a Minimal Spanning Tree (MST) corresponding to the $r$-th attribute $\mathbf{a}_r$. Each tree $M_r$ is denoted as a tuple $ M_r = <V_r, B_r, W_r> $ where all the $o_r$ possible values in $V_r$ are treated as nodes. $B_r$ is the set of $o_r-1$ edges that have the minimum sum of edge lengths and can connect all the $o_r$ nodes. $W_r$ contains the weights reflecting the edge lengths. Such a tree can be searched through the Prim or Kruskal~\cite{prim-algorithm, kruskal-algorithm} algorithm given a fully connected graph with exact edge weights as shown in Figure~\ref{fig:pipleline_figure}(a). The weights are actually the distance between any pair of possible values, which can be computed using any existing distance measure defined for qualitative values. From Figure~\ref{fig:pipleline_figure}(b), it can be seen that the constructed MST concisely and exactly reflects the local order relationship among possible values, and is thus called an order tree. 

\begin{figure}[!t]	
\centerline{\includegraphics[width=3.3in]{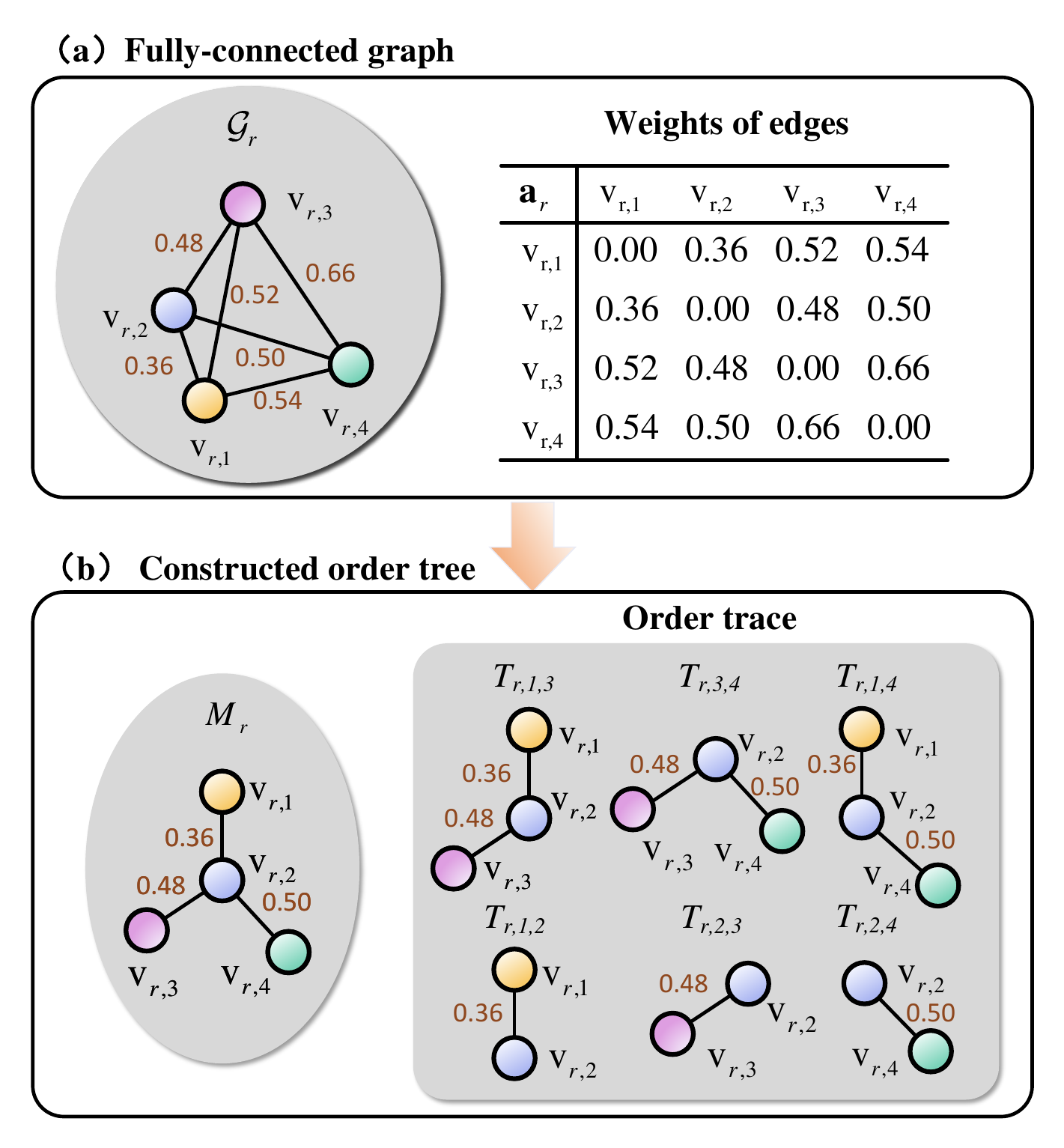}}
\caption{Process of order tree construction. (a) A fully connected graph ${\cal G}_r$ is prepared with a distance matrix reflecting the edge weights. (b) Prim or Kruskal algorithm is implemented to generate an order tree with a unique order trace between each pair of nodes, which is defined in Definition~\ref{def:trace}.}	
\label{fig:pipleline_figure}	
\end{figure}

\begin{remark}[Generalization of relationship graph]
\label{remark1}
    Given an attribute $\mathbf{a}_r$, the constructed order tree $M_r$ represents the order relationship among different possible value subsets and provides a unique trace between each pair of possible values as shown in Figure~\ref{fig:pipleline_figure}(b). Therefore, $M_r$ can be regarded as generalized from: 1) The line graph representing all the values in one order, and 2) The fully connected graph representing redundant relationships among possible values.
\end{remark}

To define dissimilarity between nodes according to the order tree, we first define the order trace between two nodes.
\begin{definition}[Order trace]\label{def:trace}
Given an order tree $M_r$, order trace \( T_{r,s,u} \) is a set containing all the weights between the nodes located on the shortest path from node \(v_{r,s}\) to \(v_{r,u}\). Since the order tree is undirected, \( T_{r,s,u}=T_{r,u,s}\) holds.
\end{definition}

It turns out that each order trace uniquely concatenates a certain number of closely connected nodes, while the other nodes that are further away are excluded. This allows the order tree to flexibly represent local order relationships of possible value subsets, thus yielding a higher degree of freedom in distance structure learning. The dissimilarity between two nodes $v_{r,s}$ and $v_{r,u}$ can be intuitively reflected by their order trace, e.g., by adding all the weights on the trace. So far, the definition of weights plays a key role in constructing order forest and forming the dissimilarity between nodes. Since we focus on the clustering task, the weights and dissimilarity are defined by sufficiently leveraging the cluster information in Section~\ref{sct:Method_OFD}, and the learning of the dissimilarity is incorporated with clustering in Section~\ref{sct:joint}. 
\vspace{-8pt}
\subsection{Clustering-Friendly Trace Distance}
\label{sct:Method_OFD}

Given cluster partition $\mathbf{Q}$, weights of a fully connected graph should be first computed and then the order tree extracted from it forms value-level dissimilarities. Specifically, weight between two nodes $v_{r,u}$ and $v_{r,s}$ is defined as the distance between their probability distributions extracted from different clusters by
\begin{equation}
    \label{eq:compute_w}
    {w_{r,u,s}} = \left\| {{{\bf{p}}_{{v_{r,u}}}} - {{\bf{p}}_{{v_{r,s}}}}} \right\|_p,
\end{equation}
where \({{\bf{p}}_{{v_{r,u}}}} = \left[ {{p_{{C_1}|{v_{r,u}}}},{p_{{C_2}|{v_{r,u}}}}, \ldots ,{p_{{C_k}|{v_{r,u}}}}} \right]\) and \({{\bf{p}}_{{v_{r,s}}}} = \left[ {{p_{{C_1}|{v_{r,s}}}},{p_{{C_2}|{v_{r,s}}}}, \ldots ,{p_{{C_k}|{v_{r,s}}}}} \right]\) are $k$-dimensional vectors representing the probability distributions of $v_{r,u}$ and $v_{r,s}$ across all the $k$ clusters. ${p_{{C_j}|{v_{r,u}}}} = {\left| {{X_{r,u}} \cap {C_j}} \right|}/{\left| {{X_{r,u}}} \right|}$ where \(| \cdot |\) counts the number of samples in a set and \({X_{r,u}} = \left\{ {{{\bf{x}}_i}|{x_{i,r}} = {v_{r,u}}} \right\}\) is a sample set collecting all the samples in \( X\) with their $r$-th values equal to $v_{r,u}$. The symbol \(\left\| \cdot \right\|_p\) represents $p$-norm, which intuitively reflects the difference in the direct probability distribution of nodes, where we adopt a common setting $p=2$. A distribution \({{\bf{p}}_{{v_{r,u}}}}\) describes the distribution pattern of a value across all the $k$ clusters, so that two values with similar patterns are considered to be more similar. We therefore use weights defined in Eq.~(\ref{eq:compute_w}) to construct the order tree that is with the likelihood of producing the current partition $\mathbf{Q}$.

With constructed order tree $M_r$ and the weight defined in Eq.~(\ref{eq:compute_w}), the dissimilarity between two possible values $v_{r,u}$ and $v_{r,s}$ is defined as the length of their order trace as defined in Definition~\ref{def:trace}, which can be written as
\begin{equation}
    \label{eq: tree-order_distance}
{d_{r,u,s}}{\rm{  = }}\sum_{w_{r,u,s}\in {T_{r,u,s}}} {w_{r,u,s}}.
\end{equation}
Since this is computed based on the weights defined in a clustering-friendly manner by Eq.~(\ref{eq:compute_w}), we thus call it clustering-friendly trace distance. Accordingly, the sample-cluster distance $\gamma(x_{i,r}, C_{j,r})$ reflected by the order tree structure $M_r$ can be defined based on the value-level trace distance as
\begin{equation}
	\label{eq:dist_r}
 \gamma(x_{i,r},C_{j,r}; M_r) ={{{\bf{p}}}_{j,r}^\top{{\bf{d}}_{r,u}}},
\end{equation}
where we assume that the sample value $x_{i,r}$ equals to the possible value $v_{r,u}$ for simplicity without loss of generality. $\mathbf{d}_{r,u}=[d_{r,u,1},d_{r,u,2},...,d_{r,u,o_r}]$ is an $o_r$-dimensional vector containing the trace distances between $v_{r,u}$ and each of the $o_r$ possible values in $V_r$, and $\mathbf{p}_{j,r}=[p_{v_{r,1}|C_j},p_{v_{r,2}|C_j},...,p_{v_{r,o_r}|C_j}]$ is an $o_r$-dimensional vector describing the probability distribution of possible values in $V_r$ within cluster $C_j$, where $p_{v_{r,u}|C_j}=|X_{r,u}\cap C_j|/|C_j|$.

Correspondingly, the overall sample-cluster distance $\Gamma(\mathbf{x}_i,C_j)$ defined based on the whole order forest $M$ can be formulated as 
\begin{equation}
	\label{eq:dist_final}
	\Gamma(\mathbf{x}_i,C_j; M)=\sum_{r=1}^{l}\gamma(x_{i,r},C_{j,r};M_r).
\end{equation}

\begin{theorem}\label{thm:d_i_r_g}
The trace distance measure $d_{r,u,s}$ defined in the context of the order tree $M_r$ represents a valid distance metric.
\end{theorem}

\begin{proof}
$d_{r,u,s}$ follows non-negativity, symmetry, and triangle inequality for any $r\in\{1,2,...,l\}$ and $u,s,g\in\{1,2,...,o_r\}$ as shown below.\\
\textbf{Non-negativity:} \(d_{r,u,s}\geq 0 \). $d_{r,u,s}$ is the length of an order trace, which is always non-negative comprising non-negative weights according to Eqs.~(\ref{eq:compute_w}) and~(\ref{eq: tree-order_distance}) with the norm set at $p=2$;\\
\textbf{Symmetry}: \(d_{r,u,s} = d_{r,s,u} \). Since the order tree is an undirected graph, the weights on the trace extracted from the undirected graph obey the commutative law for their summation; \\
\textbf{Triangle inequality}: \(d_{r,u,s} \le d_{r,u,g} + d_{r,g,s}\). The order trace is the unique path between two values with length \(d_{r,u,s}\). From $v_{r,u}$ to $v_{r,s}$, detour another node $v_{r,g}$ other than the trace $T_{r,u,s}$ necessarily involves extra weight(s) from the other traces. Given that each weight is non-negative, the result follows.
\end{proof}

\begin{theorem}\label{lem:sample_cluster_distance}
The sample-cluster distance \( \Gamma (\mathbf{x}_i, C_j; M) \) defined in the context of the order forest \( M \) represents a valid distance metric.
\end{theorem}

\begin{proof}
The computation of \( \Gamma (\mathbf{x}_i, C_j; M) \) can be viewed as the weighted sum of a series of $d_{r,u,s}$s in $\mathbf{d}_{r,u}$s with non-negative weights represented by the probabilities in $\mathbf{p}_{j,r}$s according to Eqs.~(\ref{eq:dist_r}) and~(\ref{eq:dist_final}). Since trace distance $d_{r,u,s}$ is a metric according to Theorem~\ref{thm:d_i_r_g}, \( \Gamma (\mathbf{x}_i, C_j; M) \) is also a metric following non-negativity, symmetry, and triangle inequality.
\end{proof}
\vspace{-8pt}
\subsection{Joint Learning Algorithm}\label{sct:joint}
Joint learning of cluster and order forest is facilitated by integrating the order forest construction mechanism presented in Sections~\ref{sct:Method_OFC} and \ref{sct:Method_OFD} into the clustering objective. Accordingly, $L(\mathbf{Q})$ can be refined to $L(\mathbf{Q}, M)$ based on Eqs.~(\ref{eq:obj_pre}), (\ref{eq: tree-order_distance}), (\ref{eq:dist_r}), and (\ref{eq:dist_final}):
\begin{equation}\label{eq:obj}
    L(\mathbf{Q},M)=\sum_{j=1}^k\sum_{i=1}^nq_{i,j}\cdot\sum_{r=1}^{l}\gamma(x_{i,r},C_{j,r}; M_r).
\end{equation}
Then the problem becomes how to compute $\mathbf{Q}$ and $M$ to minimize $L$, which is typically solved by iteratively fixing one and computing another. Specifically, given fixed distance structure $\hat{M}$, $\mathbf{Q}$ can be computed to minimize $L(\mathbf{Q},\hat{M})$ by
\begin{equation}
	\label{eq:qim}
	q_{i,j}=\left\{
	\begin{array}{ll}
		1,  & \text{if}\ j=\arg\min\limits_y\sum_{r=1}^{l}\gamma(x_{i,r},C_{y,r}; M_r)\\
		0,  & \text{otherwise.}\\
	\end{array}\right.
\end{equation}
with $i=\{1,2,...,n\}$ and $j=\{1,2,...,k\}$. Eq.~(\ref{eq:qim}) is strictly derived from Eq.~(\ref{eq:qim_pre}) by adopting order forest $\hat{M}$ as its distance structure. After the $\mathbf{Q}$ is computed, we fix it as $\hat{\mathbf{Q}}$ and then reconstruct $M$ according to Figure~\ref{fig:pipleline_figure} and Eqs.~(\ref{eq:compute_w}) - (\ref{eq: tree-order_distance}).

In summary, $L$ is optimized by iteratively solving the two minimization problems: 1) Fix $\hat{M}$, run $k$-modes~\cite{kmd} to iteratively compute $\mathbf{Q}$ until convergences; 2) Fix $\hat{\mathbf{Q}}$, reconstruct $M$ to update the distance metric. With a finite state space of $M$, the states will gradually be exhausted during the iterative searching, and thus the convergence of the algorithm can be guaranteed. The whole algorithm is summarized as Algorithm~\ref{alg:COForest}.
\begin{algorithm}[!t]
    \label{alg:COForest}
    \SetAlgoLined 
    \caption{\footnotesize{COForest: Clustering with Order Forest Learning}}
    \SetKwInOut{Require}{Require}
    \SetKwInOut{Ensure}{Ensure}
    \Require {Dataset $X$, number of sought clusters $k$}
    \Ensure {Partition $\mathbf{Q}$, order forest $M$}
    Initialization: Set outer and inner loop counters by ${\cal E}\leftarrow 0$ and ${\cal I}\leftarrow 0$; Run $k$-modes~\cite{kmd} to obtain a relatively stable initial $\mathbf{Q}^{\{\cal E\}}$; Construct initial $M^{\{\cal E\}}$ according to $\mathbf{Q}^{\{\cal E\}}$; 
    Set convergence mark for outer loop by $Conv\_{\cal E}\leftarrow False$.\\
    \While{$Conv\_{\cal E}=False$}{
        $Conv\_{\cal I}\leftarrow False$;\\
    \While{$Conv\_{\cal I}=False$}{
        ${\cal I}\leftarrow {\cal I}+1$; Compute $\mathbf{Q}^{\{{\cal I}\}}$ by Eq.~(\ref{eq:qim});\\    
        \If {$\mathbf{Q}^{\{{\cal I}\}}=\mathbf{Q}^{\{{\cal I}-1\}}$}{
                $Conv\_{\cal I}\leftarrow True$;
               }
               }
        \eIf {$\mathbf{Q}^{\{{\cal E}\}}=\mathbf{Q}^{\{{\cal I}\}}$}{$Conv\_{\cal E}\leftarrow True$;}{
        ${\cal E}\leftarrow {\cal E}+1$; $\mathbf{Q}^{\{{\cal E}\}}\leftarrow \mathbf{Q}^{\{{\cal I}\}}$; Reconstruct $M^{\{{\cal E}\}}$; 
        }
    }
\end{algorithm}

\begin{theorem}
    \label{theorem:time_complex}
	Time complexity of COForest is $O(nlk{\cal I}{\cal E})$
\end{theorem}
\textit{Proof.} To more intuitively provide the proof, we first define the probability $\mathbf{P}$ and trace distance matrix $\mathbf{D}$, and assume that $\varsigma=\max(o_1,o_2,...,o_l)$ for worst-case analysis. $\mathbf{P}$ is a \(k \times l\) probability matrix with its \((j,r)\)-th entry $\mathbf{p}_{j,r}=[p_{v_{r,1}|C_j},p_{v_{r,2}|C_j},...,p_{v_{r,o_r}|C_j}]$. \(\mathbf{D}\) is an \(l \times \varsigma\) trace distance matrix, and its \((r,u)\)-th entry is $\mathbf{d}_{r,u}=[d_{r,u,1},d_{r,u,2},...,d_{r,u,o_r}]$.

Assume solving problem $L(\mathbf{Q},\hat{M})$ involves ${\cal I}$ iterations to compute $\mathbf{Q}$ and $\mathbf{P}$, and the whole algorithm involves ${\cal E}$ iterations to construct $M$ and update $\mathbf{D}$ for solving $L(\hat{\mathbf{Q}},M)$. 

For each iteration of ${\cal I}$, $\mathbf{P}$ should be prepared by going through all the $n$ data samples once with complexity $O(nkl\varsigma)$, and $\mathbf{D}$ should be prepared by going through all the $\varsigma$ values of $l$ attributes once on $n$ samples with complexity $O(nl\varsigma)$. Then the $n$ samples are clustered to $k$ clusters by considering $\varsigma$ values of $l$ attributes according to Eq.~(\ref{eq:qim}), with time complexity $O(nkl\varsigma)$. Therefore, the time complexity of solving $L(\mathbf{Q},\hat{M})$ in a total of ${\cal I}$ iterations is $O({\cal I}nkl\varsigma)$.

For each iteration of ${\cal E}$, since $\mathbf{P}$ and $\mathbf{D}$ have been prepared, order tree of $\varsigma$ possible values of each of the $l$ attributes can be searched by constructing $M_{r}$ with time complexity $O(n l \varsigma^2)$. For ${\cal E}$ iterations of the whole COForest algorithm, considering the ${\cal I}$ inner iterations, the overall time complexity of COForest is $O({\cal E}({\cal I}nkl\varsigma+n l \varsigma^2))$.

Since $\varsigma$ is a small integer ranging from 2 to 8 in most cases, it can be treated as a constant, and the overall time complexity can be simplified to $O(nlk{\cal I}{\cal E})$, which is linear to $n$ and $l$. \hfill\Square

\vspace{-8pt}
\section{Experiments}
\label{sct:EXP}

Five experiments are designed to evaluate the proposed COForest by comparing it with 10 counterparts on 12 real public datasets using three validity metrics. The experiments are  summarized below:
\vspace{-5pt}
\begin{itemize}
\item Clustering performance comparisons with significance tests illustrate that COForest significantly outperforms the conventional and state-of-the-art counterparts (Section~\ref{sct: cpAsig}).
\item Ablation studies comparing five ablated versions of COForest confirm the effectiveness of each of the core components of COForest (Section~\ref{sct: ablation}). 
\item Convergence and efficiency of COForest are demonstrated by plotting the objective function values during learning and execution time under different dataset scales, respectively (Section~\ref{sct: cAe}).
\item Reasonableness of the learned distance structure is well confirmed by qualitatively comparing the cluster discrimination ability of different methods using t-SNE (Section~\ref{sct: qa}).
\item The potential of extending COForest to mixed data with numerical and categorical attributes is validated by comparing its clustering performance with those specifically proposed for mixed data (please refer to the ``Supplementary Material'' provided by~\cite{code}). 
\end{itemize}
\vspace{-8pt}
\subsection{Experimental Setup}
\label{sct: setup}

\begin{table}[!t]
\caption{Information of the 10 counterparts. ``Type'' indicates whether a method separates or jointly learns the distance definition and clustering.}
\label{tb:method_statistics}
\centering
\begin{tabular}{c|rcc}
\toprule
No. &Counterpart& Year & Type\\
\midrule
1& KMD \cite{kmd} & 1998 & Separate \\
2& LSM \cite{lsm} & 1998 & Separate \\
3& JDM \cite{jdm} & 2016 & Separate \\
4& CBDM \cite{cbdm_journal} & 2012 & Separate \\
5& OCIL \cite{ocil} & 2013 & Joint \\
6& UDMC \cite{udm} & 2022 & Separate \\
7& DLC \cite{dlc} & 2020 & Joint \\
8& H2H \cite{het2hom} & 2022 & Joint \\
9& HDC \cite{HDC} & 2022 & Joint \\
10& ADC \cite{adc} & 2023 & Separate \\
\bottomrule
\end{tabular}
\end{table}

Experimental settings are briefly described below.

\textbf{10 Counterparts} are sorted out in Table~\ref{tb:method_statistics}. 
We set their hyper-parameters (if any) to the values recommended by the corresponding papers. Each method is implemented 10 times and the average performance is reported.

\textbf{12 Datasets} from various domains are utilized for the experiments. All the datasets are real public datasets collected from the UCI Machine Learning Repository~\cite{uci}, and the statistical information is shown in Table~\ref{tb:statistics}. Before the experiments, we preprocess the datasets by removing samples with missing values. Since we focus on categorical data clustering, numerical attributes in AC, TS, HF, and DS datasets are omitted. For all the compared methods, we set $k=k^*$ as the sought number of clusters. 

\begin{table}[!t]
\caption{Statistics of the 12 datasets. $l$ and $n$ are the numbers of attributes and samples, respectively. $k^*$ is the true number of clusters.}
\label{tb:statistics}
\centering
\begin{tabular}{r|cc|ccc}
\toprule
No. & Dataset & Abbrev. & $l$ & $n$ \ \ \ & $k^*$ \\
\midrule
1& Hayes-Roth & HR &4 & 132 & 3 \\
2& Car Evaluation & CE & 6 & 1728 & 4\\
3& Australia Credit & AC & 8 & 690 & 2 \\
4& Congressional Voting & VT & 16 & 435 & 2 \\	
5& Caesarian Section & CS & 4 & 80 & 2 \\	
6& Soybean (small) & SB &35 & 47 & 4 \\	
7& Nursery School & NS & 8 & 12960 & 4 \\
8& Zoo & ZO &16 & 101 & 7 \\		
9& Thoracic Surgery & TS & 13 & 470 & 2 \\	
10& Heart Failure & HF &5& 299 & 2 \\
11& Inflammations Diagnosis & DS &5 & 120 & 2 \\	 
12& Lenses & LS & 4 & 24 & 3\\
\bottomrule
\end{tabular}
\end{table}

\textbf{Three Evaluation Metrics} include the clustering accuracy (CA)~\cite{ex1}, Adjusted Rand Index (ARI)~\cite{ex3, ex4}, and Normalized Mutual Information (NMI)~\cite{ex5}, are adopted for evaluating clustering performance from different perspectives. Among them, CA is a conventional index, which computes the matching rate based on the best permutation mapping between the obtained clusters and the true classes. In contrast, ARI and NMI are more discriminative, being in value intervals \([-1,1]\), and \([0,1]\), respectively. For all the indices, a higher value indicates a better clustering performance. NMI results
are provided in the ``Supplementary Material''~\cite{code}.
\vspace{-8pt}

\subsection{Clustering Performance}
\label{sct: cpAsig}
In this section, we investigate the clustering performance of different algorithms and statistically analyze the superiority of COForest.

\begin{table*}[!t]
\caption{Clustering performance evaluated by CA. ``$\overline{\textit{AR}}$'' row reports the average performance rankings.}
\label{tb:clustering_ca}
\centering
\resizebox{2.06\columnwidth}{!}{
\begin{tabular}{l|ccccccccccc}
\toprule
Data & KMD &LSM& JDM & CBDM & OCIL & UDMC & DLC & H2H & HDC & ADC & COForest (ours)\\
\midrule

HR& ~0.3795$\pm$0.02	& ~0.3826$\pm$0.03	& ~0.3841$\pm$0.03	& ~\underline{0.4083$\pm$0.06}	& ~0.3621$\pm$0.05	& ~0.3886$\pm$0.01	& ~0.3659$\pm$0.03	& ~0.3333$\pm$0.00	& ~0.3758$\pm$0.02	& ~0.3970$\pm$0.05	& ~\textbf{0.4530$\pm$0.07}\\	
CE& ~0.3730$\pm$0.04	& ~0.3587$\pm$0.04	& ~0.3597$\pm$0.04	& ~-	& ~0.3659$\pm$0.05	& ~0.3505$\pm$0.03	& ~\underline{0.3746$\pm$0.04}	& ~0.3354$\pm$0.06	& ~0.3730$\pm$0.04	& ~0.3730$\pm$0.04	& ~\textbf{0.4261$\pm$0.06}\\
AC& ~0.7494$\pm$0.05	& ~0.7823$\pm$0.04	& ~0.6858$\pm$0.12	& ~0.7417$\pm$0.08	& ~0.7781$\pm$0.10	& ~0.7674$\pm$0.08	& ~0.7499$\pm$0.14	& ~\underline{0.7942$\pm$0.00}	& ~0.7484$\pm$0.09	& ~0.7709$\pm$0.09	& ~\textbf{0.8307$\pm$0.05}\\	
VT& ~0.8621$\pm$0.01	& ~0.8662$\pm$0.00	& ~0.8662$\pm$0.00	& ~0.8749$\pm$0.00	& ~\textbf{0.8763$\pm$0.00}	& ~0.8639$\pm$0.00	& ~0.8540$\pm$0.08	& ~0.8736$\pm$0.00	& ~0.8736$\pm$0.00	& ~0.8713$\pm$0.00	& ~\underline{0.8761$\pm$0.00}\\	
CS& ~0.5475$\pm$0.02	& ~0.5425$\pm$0.05	& ~0.5475$\pm$0.05	& ~0.5787$\pm$0.03	& ~0.5037$\pm$0.18	& ~0.5788$\pm$0.03	& ~0.6013$\pm$0.04	& ~\underline{0.6050$\pm$0.02}	& ~0.5862$\pm$0.03	& ~0.5875$\pm$0.02	& ~\textbf{0.6450$\pm$0.02}\\	
SB& ~0.8191$\pm$0.18	& ~0.8553$\pm$0.19	& ~0.7830$\pm$0.16	& ~0.8213$\pm$0.15	& ~0.7936$\pm$0.33	& ~0.8426$\pm$0.17	& ~0.8723$\pm$0.17	& ~\underline{0.9511$\pm$0.10}	& ~0.8128$\pm$0.13	& ~0.8191$\pm$0.16	& ~\textbf{0.9723$\pm$0.09}\\	
NS& ~0.3454$\pm$0.04	& ~0.3171$\pm$0.04	& ~0.3064$\pm$0.03	& -	& ~\underline{0.3454$\pm$0.08}	& ~0.3235$\pm$0.04	& ~0.3301$\pm$0.06	& ~0.3441$\pm$0.05	& ~0.3454$\pm$0.04	& ~0.3454$\pm$0.04	& ~\textbf{0.3626$\pm$0.09}\\	
ZO& ~0.6564$\pm$0.10	& ~0.6594$\pm$0.10	& ~\underline{0.7149$\pm$0.09}	& ~0.6921$\pm$0.08	& ~0.5663$\pm$0.31	& ~0.6564$\pm$0.09	& ~0.7020$\pm$0.10	& ~0.6980$\pm$0.04	& ~0.6713$\pm$0.12	& ~0.6812$\pm$0.11	& ~\textbf{0.7832$\pm$0.12}\\	
TS& ~0.7083$\pm$0.08	& ~0.6717$\pm$0.08	& ~0.7023$\pm$0.10	& ~0.6957$\pm$0.09	& ~0.6689$\pm$0.08	& ~0.7087$\pm$0.09	& ~0.6868$\pm$0.08	& ~0.5723$\pm$0.03	& ~\underline{0.7104$\pm$0.08}	& ~0.6947$\pm$0.10	& ~\textbf{0.7232$\pm$0.09}\\	
HF& ~0.5344$\pm$0.03	& ~0.5344$\pm$0.03	& ~0.5421$\pm$0.03	& ~\underline{0.5498$\pm$0.02}	& ~0.4880$\pm$0.17	& ~0.5378$\pm$0.02	& ~0.5381$\pm$0.02	& ~0.5441$\pm$0.05	& ~0.5388$\pm$0.02	& ~0.5378$\pm$0.02	& ~\textbf{0.5532$\pm$0.03}\\	
DS& ~0.6833$\pm$0.11	& ~0.6833$\pm$0.11	& ~0.6975$\pm$0.11	& ~0.7142$\pm$0.12	& ~\underline{0.7242$\pm$0.16}	& ~0.6725$\pm$0.11	& ~\underline{0.7242$\pm$0.16}	& ~0.6267$\pm$0.04	& ~0.6725$\pm$0.11	& ~0.6975$\pm$0.11	& ~\textbf{0.7617$\pm$0.08}\\	
LS& ~0.5250$\pm$0.07	& ~0.5417$\pm$0.10	& ~0.5417$\pm$0.10	& -	& ~0.5417$\pm$0.08	& ~\underline{0.5792$\pm$0.14}	& ~0.5500$\pm$0.09	& ~0.5167$\pm$0.09	& ~0.5250$\pm$0.07	& ~0.5250$\pm$0.07	& ~\textbf{0.6833$\pm$0.14}\\	
		
\midrule	
$\overline{\textit{AR}}$ & 7.2500    &7.2083    &6.7917    &6.4583    &7.0833    &6.5833    &5.3750    &6.3750    &6.2083    &5.5833    &1.0833 \\	
\bottomrule
\end{tabular}}
\end{table*}

\begin{table*}[!t]
\caption{Clustering performance evaluated by ARI. ``$\overline{\textit{AR}}$'' row reports the average performance rankings.}
\label{tb:clustering_ari}
\centering
\resizebox{2.06\columnwidth}{!}{
\begin{tabular}{l|ccccccccccc}
\toprule
Data & KMD &LSM& JDM & CBDM & OCIL & UDMC & DLC & H2H & HDC & ADC & COForest (ours)\\
\midrule
HR& -0.0064$\pm$0.01	& -0.0051$\pm$0.01	& -0.0048$\pm$0.01	& ~\underline{0.0127$\pm$0.03}	& -0.0073$\pm$0.02	& -0.0037$\pm$0.00	& -0.0068$\pm$0.01	& -0.0149$\pm$0.00	& -0.0056$\pm$0.01	& ~0.0043$\pm$0.03	& ~\textbf{0.0429$\pm$0.04}\\	
CE& ~0.0229$\pm$0.03	& ~0.0314$\pm$0.02	& ~0.0321$\pm$0.02	& -	& ~0.0501$\pm$0.06	& ~0.0289$\pm$0.02	& ~\underline{0.0676$\pm$0.03}	& ~0.0140$\pm$0.03	& ~0.0229$\pm$0.03	& ~0.0229$\pm$0.03	& ~\textbf{0.1016$\pm$0.07}\\		
AC& ~0.2575$\pm$0.10	& ~0.3228$\pm$0.07	& ~0.1892$\pm$0.17	& ~0.2569$\pm$0.11	& ~0.3421$\pm$0.15	& ~0.3107$\pm$0.11	& ~0.3178$\pm$0.22	& ~\underline{0.3453$\pm$0.00}	& ~0.2714$\pm$0.12	& ~0.3225$\pm$0.12	& ~\textbf{0.4462$\pm$0.12}\\	
VT& ~0.5233$\pm$0.02	& ~0.5354$\pm$0.01	& ~0.5354$\pm$0.01	& ~0.5613$\pm$0.01	& ~\textbf{0.5655$\pm$0.01}	& ~0.5287$\pm$0.01	& ~0.5208$\pm$0.18	& ~0.5572$\pm$0.00	& ~0.5572$\pm$0.00	& ~0.5503$\pm$0.00	& ~\underline{0.5647$\pm$0.00}\\	
CS& -0.0033$\pm$0.01	& ~0.0017$\pm$0.03	& ~0.0038$\pm$0.03	& ~0.0137$\pm$0.01	& ~0.0070$\pm$0.03	& ~0.0140$\pm$0.02	& ~\underline{0.0342$\pm$0.03}	& ~0.0319$\pm$0.02	& ~0.0191$\pm$0.02	& ~0.0190$\pm$0.01	& ~\textbf{0.0732$\pm$0.02}\\	
SB& ~0.7657$\pm$0.20	& ~0.8164$\pm$0.24	& ~0.6826$\pm$0.22	& ~0.7652$\pm$0.21	& ~0.7902$\pm$0.33	& ~0.8232$\pm$0.19	& ~0.8500$\pm$0.20	& ~\underline{0.9271$\pm$0.16}	& ~0.7595$\pm$0.18	& ~0.7863$\pm$0.19	& ~\textbf{0.9562$\pm$0.14}\\	
NS& ~0.0630$\pm$0.02	& ~0.0556$\pm$0.02	& ~0.0457$\pm$0.02	& -	& ~\underline{0.1146$\pm$0.10}	& ~0.0617$\pm$0.03	& ~0.0886$\pm$0.08	& ~0.0847$\pm$0.09	& ~0.0630$\pm$0.02	& ~0.0630$\pm$0.02	& ~\textbf{0.1352$\pm$0.13}\\	
ZO& ~0.5707$\pm$0.13	& ~0.5872$\pm$0.15	& ~\underline{0.6496$\pm$0.14}	& ~0.6187$\pm$0.12	& ~0.5093$\pm$0.29	& ~0.5937$\pm$0.15	& ~0.6315$\pm$0.12	& ~0.6255$\pm$0.06	& ~0.6010$\pm$0.15	& ~0.6128$\pm$0.15	& ~\textbf{0.7511$\pm$0.18}\\	
TS& ~0.0054$\pm$0.05	& ~0.0123$\pm$0.05	& ~0.0188$\pm$0.05	& ~0.0171$\pm$0.04	& -0.0048$\pm$0.05	& ~\underline{0.0198$\pm$0.05}	& -0.0034$\pm$0.04	& -0.0249$\pm$0.00	& ~0.0084$\pm$0.05	& ~0.0031$\pm$0.04	& ~\textbf{0.0220$\pm$0.04}\\	
HF& -0.0067$\pm$0.00	& -0.0067$\pm$0.00	& -0.0013$\pm$0.01	& -0.0009$\pm$0.01	& \underline{-0.0002$\pm$0.00}	& -0.0023$\pm$0.00	& -0.0005$\pm$0.00	& -0.0043$\pm$0.00	& -0.0019$\pm$0.00	& -0.0023$\pm$0.00	& ~\textbf{0.0045$\pm$0.01}\\	
DS& ~0.1697$\pm$0.18	& ~0.1697$\pm$0.18	& ~0.1944$\pm$0.19	& ~0.2280$\pm$0.20	& ~\underline{0.2839$\pm$0.32}	& ~0.1543$\pm$0.18	& ~\underline{0.2839$\pm$0.32}	& ~0.0615$\pm$0.04	& ~0.1543$\pm$0.18	& ~0.1944$\pm$0.19	& ~\textbf{0.2901$\pm$0.16}\\	
LS& ~0.0756$\pm$0.10	& ~0.1180$\pm$0.15	& ~0.1180$\pm$0.15	& -	& ~0.1287$\pm$0.12	& ~\underline{0.1919$\pm$0.22}	& ~0.1379$\pm$0.15	& ~0.0786$\pm$0.11	& ~0.0756$\pm$0.10	& ~0.0756$\pm$0.10	& ~\textbf{0.3359$\pm$0.22}\\	
\midrule	
$\overline{\textit{AR}}$ & 8.6667    &7.0000    &6.5417    &6.7500    &5.2083    &6.0833    &4.7083    &6.5417    &7.0833    &6.3333    &1.0833
 \\	
\bottomrule
\end{tabular}}
\end{table*}

\textbf{Clustering performance of different methods} are compared in Tables \ref{tb:clustering_ca} and \ref{tb:clustering_ari} w.r.t. CA and ARI, respectively. The best and second-best results on each dataset are highlighted in \textbf{bold} and \underline{underline}, respectively. The observations include the following three aspects: 1) Overall, COForest performs best on almost all datasets, indicating its superiority in clustering. 2) The performance of COForest on the TS and HF datasets is not obviously better than the second-best method. However, the second-best method varies on these datasets, indicating the robustness of COForest. 3) Although COForest does not have the best CA and ARI performance on the VT dataset, it maintains the second-best and is not surpassed by much by the winners. In addition, the results of CBDM on CE, NS, and LS datasets are not reported because the attributes of these datasets are independent of each other, making CBDM fails in measuring distances according to the correlated attributes.

\textbf{Significance tests} are conducted by first implementing Friedman tests on the average performance ranks reported in the last rows in Tables~\ref{tb:clustering_ca} and~\ref{tb:clustering_ari}, respectively. The corresponding p-values are \(0.00020\) and \(0.00002\), respectively, both passing the test under 99\% confidence interval (i.e., p-value = 0.01). On this basis, Bonferroni Dunn (BD) post-hoc tests are implemented. Critical Difference (CD) intervals for the two-tailed BD tests at 95\% ($\alpha$= 0.05) and 90\% ($\alpha$ = 0.1) confidence intervals are 3.8048 and 3.5204, respectively, for comparing 11 methods across 12 datasets. As can be seen from the ``$\overline{\textit{AR}}$'' rows in Tables~\ref{tb:clustering_ca} and~\ref{tb:clustering_ari} that all compared methods fall outside the right boundary of the CD intervals, except for the DLC method w.r.t. ARI performance under $\alpha=0.05$. But it is worth mentioning that DLC is very close to the boundary of $\alpha=0.05$ and stays outside the boundary of $\alpha=0.1$. In general, the test results indicate that the proposed COForest significantly outperforms the other counterparts.
\vspace{-8pt}
\subsection{Ablation Study}
\label{sct: ablation}

\begin{figure}[!t]	
\begin{minipage}{0.32\linewidth}	
  \centerline{\includegraphics[width=1.1in]{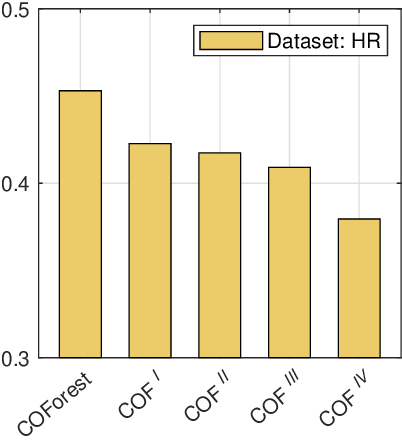}}	
\end{minipage}	
\hfill	
\begin{minipage}{0.32\linewidth}	
  \centerline{\includegraphics[width=1.1in]{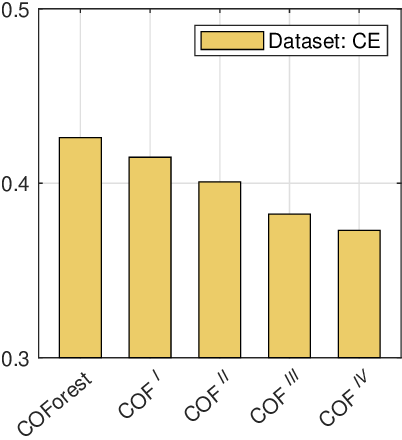}}	
\end{minipage}	
\hfill	
\begin{minipage}{0.32\linewidth}	
  \centerline{\includegraphics[width=1.1in]{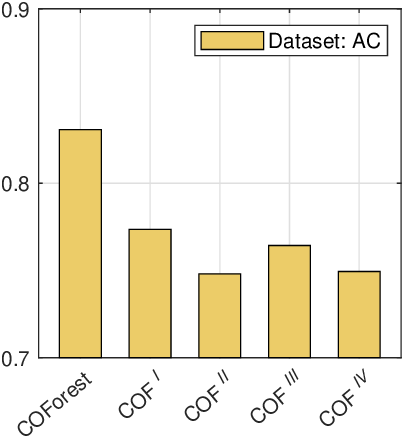}}	
\end{minipage}	
\vfill	
\begin{minipage}{0.32\linewidth}	
  \centerline{\includegraphics[width=1.1in]{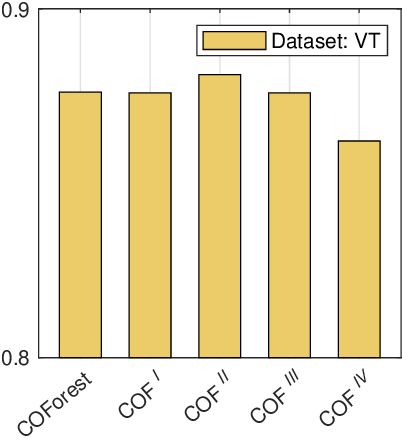}}	
\end{minipage}	
\hfill	
\begin{minipage}{0.32\linewidth}	
  \centerline{\includegraphics[width=1.1in]{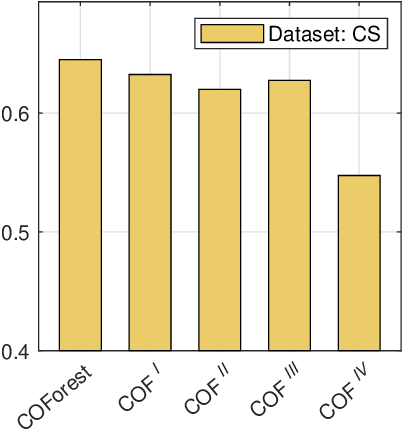}}	
\end{minipage}	
\hfill	
\begin{minipage}{0.32\linewidth}	
  \centerline{\includegraphics[width=1.1in]{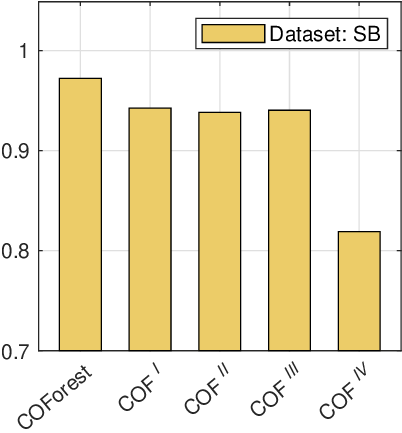}}	
\end{minipage}	
\vfill	
\begin{minipage}{0.32\linewidth}	
  \centerline{\includegraphics[width=1.1in]{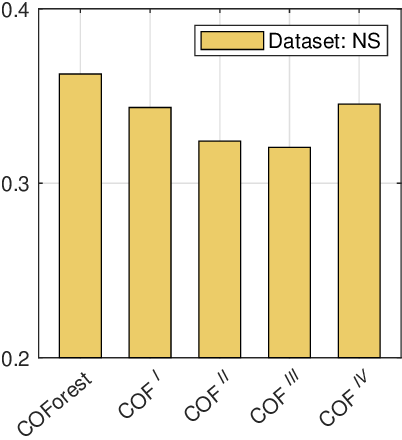}}	
\end{minipage}	
\hfill	
\begin{minipage}{0.32\linewidth}	
  \centerline{\includegraphics[width=1.1in]{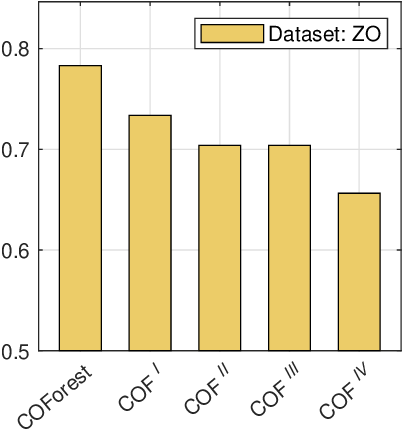}}	
\end{minipage}	
\hfill	
\begin{minipage}{0.32\linewidth}	
  \centerline{\includegraphics[width=1.1in]{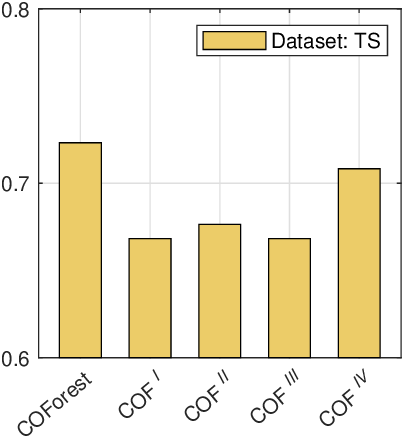}}	
\end{minipage}	
\vfill	
\begin{minipage}{0.32\linewidth}	
  \centerline{\includegraphics[width=1.1in]{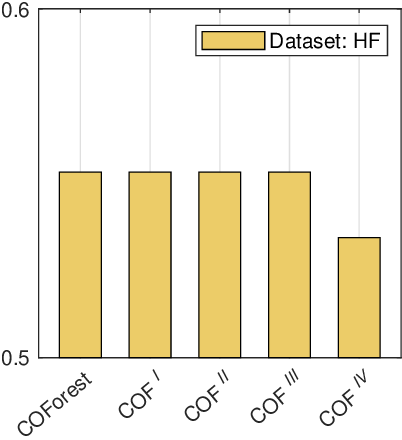}}	
\end{minipage}	
\hfill	
\begin{minipage}{0.32\linewidth}	
  \centerline{\includegraphics[width=1.1in]{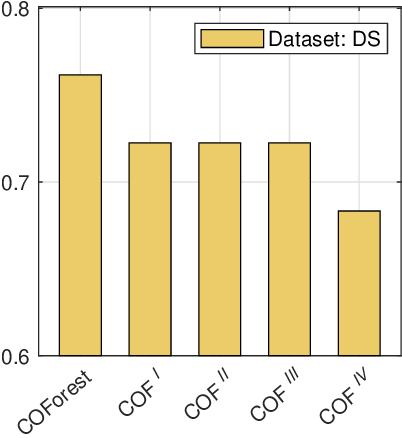}}	
\end{minipage}	
\hfill	
\begin{minipage}{0.32\linewidth}	
  \centerline{\includegraphics[width=1.1in]{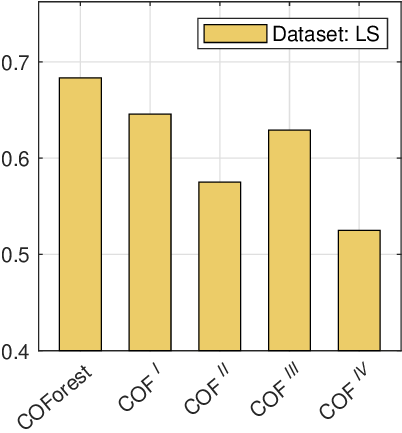}}	
\end{minipage}	
\caption{CA performance of different ablated COForest versions.}	
\label{fig:ablation_mix3}	
\end{figure}

To explicitly demonstrate the effectiveness of the core components of COForest, several ablated versions of it are compared in Figure~\ref{fig:ablation_mix3}. To evaluate the proposed order forest learning mechanism, we compare COForest with COF$^I$, which constructs the order forest once without iterative learning. To evaluate the proposed order forest structure, COF$^I$ is modified by replacing the order forest with line graphs and fully connected graphs to form COF$^{II}$ and COF$^{III}$, respectively. Moreover, to verify the adopted probability distribution-based measure in Eq.~(\ref{eq:compute_w}) for weights computing, we further let COF$^{III}$ adopt the traditional Hamming distance and form the version COF$^{IV}$.

It can be observed from Figure~\ref{fig:ablation_mix3} that COForest outperforms its four variants, which generally illustrates its effectiveness. More specific observations are four-fold: 1) COForest performs not worse than COF$^I$ on all the datasets, validating the necessity of the joint learning of the order forest and clustering. 2) On 10 out of 12 datasets, the performance of COF$^I$ is not worse than COF$^{II}$ and COF$^{III}$. This indicates that our constructed order forest is more reasonable in reflecting the distance structures, even without learning. The reason would be that the order tree is a generalized distance structure as analyzed in Remark~\ref{remark1}, which can more flexibly represent multiple local order relationships. 3) The mutual win and loss of COF$^{II}$ and COF$^{III}$ across the 12 datasets reveals that both line graph and fully connected graph have their own limitations. 4) COF$^{III}$ adopting probability distribution-based measure outperforms COF$^{IV}$ adopting Hamming distance on 10 datasets, indicating that the use of the probability distributions in Eq.~(\ref{eq:compute_w}) is reasonable.

\subsection{Convergence and Efficiency Evaluation}
\label{sct: cAe}

\begin{figure}[!t]	
\centerline{\includegraphics[width=3.6in]{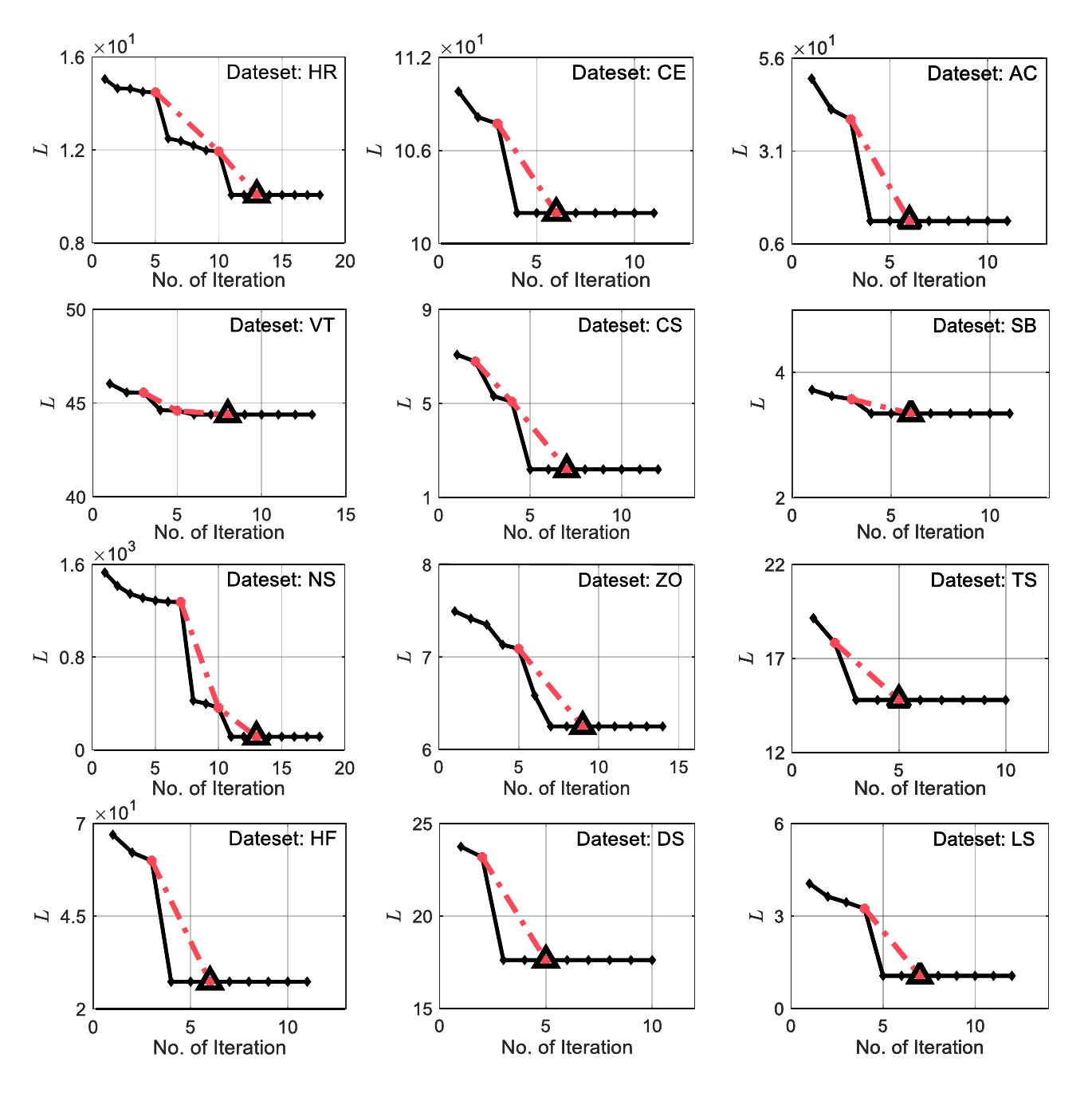}}
\caption{Convergence curves of COForest on different datasets. $L$ represents the value of the objective function. }	
\label{fig:conv_curve}	
\end{figure}

To evaluate the convergence of COForest, we plot its objective function values $L$ during the learning on all the 12 datasets in Figure \ref{fig:conv_curve}. The horizontal and vertical axes represent the number of learning iterations and the value of $L$, respectively. The triangle markers on the curve represent the iterations that COForest converges and the red dots mark the iterations of order forest reconstruction. It can be observed that, after each update of the order forest, $L$ decreases, indicating that the forest reconstruction is consistent with the minimization of $L$. Moreover, COForest converges within 15 iterations in most cases, which is quite efficient for a learning process that iteratively reconstructs the distance structure and learning data partitions.

\begin{figure}[!t]	
\centerline{\includegraphics[width=3.6in]{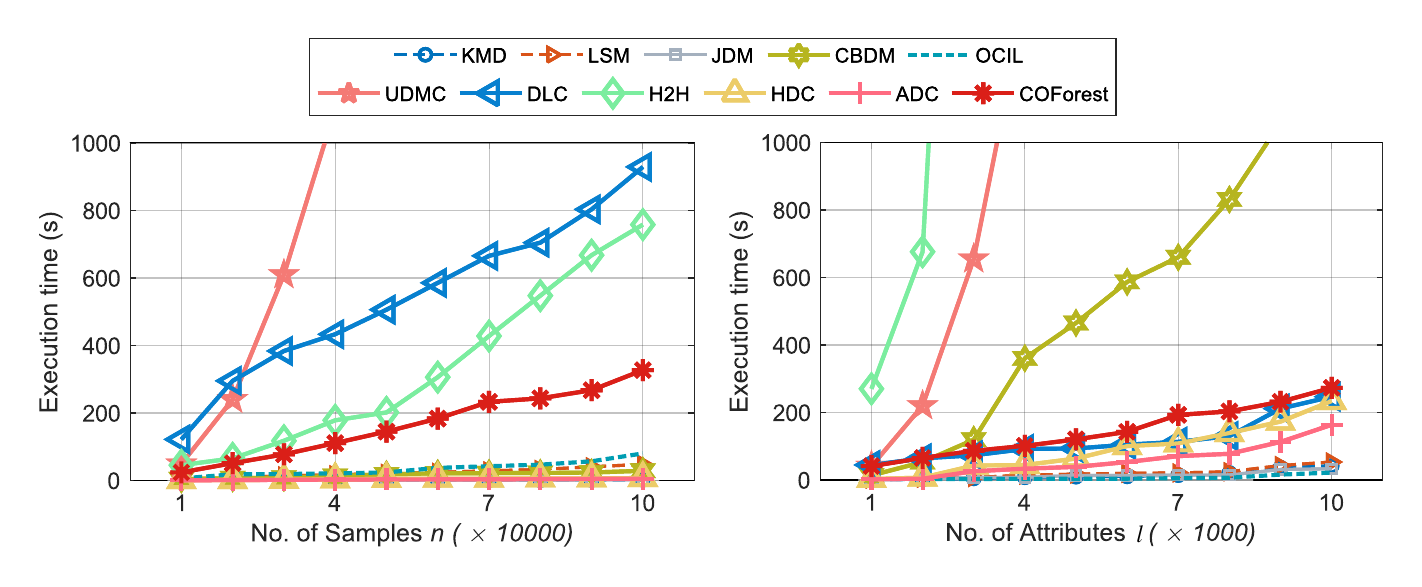}}
\caption{Execution time on synthetic datasets.}	
\label{fig:time_complex}	
\end{figure}	

To evaluate the efficiency of COForest, large synthetic datasets are randomly generated with different scales of attributes and samples. Specifically, we generate by: 1) Fixing the number of attributes at $l=20$ and increasing the number of samples $n$ from 10k to 100k with step-size 10k, and 2) Fixing sample size at $n=2k$ and increasing the number of attributes $l$ from 1k to 10k with step-size 1k, where `k' indicates `kilo'. Note that each attribute has five possible values, and the number of clusters $k$ is consistently set to five. The execution time of all the 11 methods is demonstrated in Figure~\ref{fig:time_complex}. It can be seen that the execution time of COForest is lower than or similar to the state-of-the-art UDMC, DLC, and H2H. Moreover, the increasing trend of the execution time of COForest is almost linear with $n$ and $l$, which is consistent with the time complexity analysis of Theorem~\ref{theorem:time_complex}. In summary, COForest is efficient compared to the state-of-the-art methods and does not incur too much additional computational cost compared to the simplest methods.
\vspace{-8pt}

\subsection{Qualitative Evaluation}
\label{sct: qa}

\begin{figure}[!t]	
\subfigure[CBDM $@$ AC]{
\begin{minipage}{0.27\linewidth}	
  \centerline{\includegraphics[width=1in]{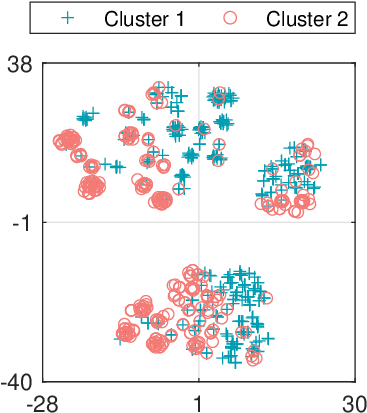}}	
\end{minipage}	
}
\hfill	
\subfigure[ADC $@$ AC]{
\begin{minipage}{0.27\linewidth}	
  \centerline{\includegraphics[width=1in]{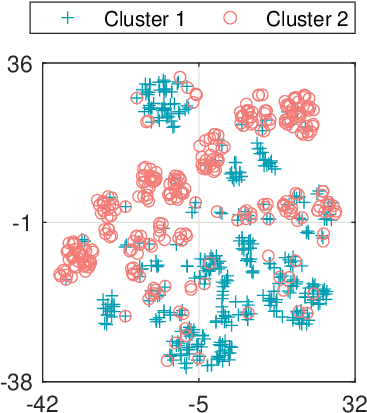}}	
\end{minipage}	
}
\hfill	
\subfigure[COForest $@$ AC]{
\begin{minipage}{0.27\linewidth}	
  \centerline{\includegraphics[width=1in]{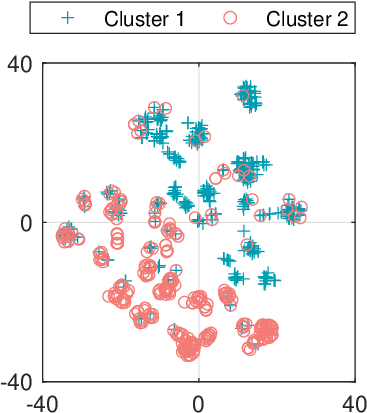}}	
\end{minipage}	
}
\caption{t-SNE visualization of the AC dataset.}	
\label{fig:Visualization_tsne}	
\end{figure}

To illustrate the cluster discrimination capability of COForest and the intuitiveness of the distance structure it obtains, we use the distance between attribute values learned by COForest, CBDM, and ADC to encode the attributes of the AC dataset. The encoded data are then dimensionally reduced into a 2-D space through t-SNE~\cite{r1visual} and visualized in Figure \ref{fig:Visualization_tsne} by marking the data points with `true' labels provided by the dataset. If more data points with the same label are gathered in the visualization, then it indicates that the corresponding distance metric is more competent in discriminating different clusters. It can be seen that COForest has significantly better cluster discrimination ability in the comparisons, which indicates the intuitiveness of its obtained clusters upon the tree-like distance structure.
\vspace{-8pt}
\section{Relate Work}
\label{sct:relate_work}
This section overviews the existing distance-measure-based and distance-learning-based categorical data clustering methods.

\textbf{Distance Measures} for categorical data including the measures yielded by encoding strategies and the directly defined distance measures. Traditional data encoding techniques, such as one-hot encoding, use Hamming distances to encode each possible value into a new attribute. However, it fails to capture the full spectrum of dissimilarity between possible values due to its boolean nature~\cite{intro10, intro6}. To overcome these limitations, statistical-based measures have been introduced that consider the frequency of intra-attribute values, thereby capturing lower information entropy for similar values and suggesting more reasonable distance metrics~\cite{lsm, ocil}. Further advancements in this area have developed metrics that account for inter-attribute dependencies, providing a more holistic view of the data relationships~\cite{adm, abdm, cbdm_journal}. Additionally, consideration of value order differentiates between nominal and ordinal attributes, with specific approaches defining distances by integrating semantic order, thereby obviously improving distance accuracy for ordinal data~\cite{ebdmjournal, udm}.

\textbf{Distance Learning} methods incorporate the defining of distances into the learning process of clustering. This includes advanced representation learning techniques that dynamically encode categorical data. For instance, some studies use various kernels to untangle attribute couplings more effectively~\cite{untie}, while others develop mixed encoding strategies for both numerical and categorical attributes~\cite{mai, mix2vec}. These often require meticulous tuning of hyperparameters. A significant step forward in this domain is the introduction of parameter-free approaches that learn the optimal number of clusters and the distances~\cite{woc,ICDCS'24}. Another innovative strategy involves transforming nominal attributes into ordinal ones through geometric projections to optimize latent distances, significantly enhancing clustering performance by unifying the treatment of different types of categorical attributes~\cite{het2hom}. Furthermore,~\cite{dlc} treats values with order information as line graphs and learns the graph weights. Later, the works~\cite{HDC, adc} further unify the distances of nominal and ordinal attributes and make them learnable with clustering.

These advances significantly improve clustering performance on categorical data. Nevertheless, coupled thorny problems still lie ahead: \textit{reasonable prior knowledge} is the premise of effective \textit{distance learning} whilst the \textit{data knowledge} is usually obtained by observing data distribution under \textit{well-defined distance metrics}.
\vspace{-8pt}

\section{Concluding Remarks}
\label{sct:Conclude}
This paper demonstrates and analyzes the key issue that bottlenecks the current qualitative data clustering performance, i.e., distance learning is restricted by prior knowledge of the distance structure. Accordingly, a new learning paradigm called COForest is proposed, which incorporates the construction of distance structure into the learning process and achieves joint optimization with clustering. Given the number of sought clusters $k$, the learning process of COForest is parameter-free and can be easily applied to various datasets. Moreover, the learned tree-like distance structures are concise and highly interpretable, making them very suitable for representing the implicit distribution of qualitative data. Extensive experiments illustrate the superiority of COForest, as well as the effectiveness of its key technical components.

The proposed method demonstrates outstanding clustering performance on static qualitative data under the given `true' number of clusters. In the future, it is promising to consider extending it to coping with more complex real situations, e.g., learning from \textit{streaming} data composed of a \textit{mixture} of quantitative and qualitative attributes with an \textit{unknown} number of \textit{imbalanced} clusters.



\section*{Acknowledgements}
This work was supported in part by the National Natural Science Foundation of China (NSFC) under grants: 62476063, 62102097, 62376233, and 62306181, the NSFC/Research Grants Council (RGC) Joint Research Scheme under the grant N\_HKBU214/21, the Natural Science Foundation of Guangdong Province under grants: 2024A1515010163 and 2023A1515012855, the General Research Fund of RGC under grants: 12201321, 12202622, and 12201323, the RGC Senior Research Fellow Scheme under grant SRFS2324-2S02, the Shenzhen Science and Technology Program under grant RCBS20231211090659101, and the Xiaomi Young Talents Program.

\bibliography{m863}
\end{document}